\documentclass[conference]{cls/IEEEtran}
\IEEEoverridecommandlockouts

\usepackage{cite}
\usepackage{hyperref}
\usepackage{amsmath,amssymb,amsfonts}
\usepackage{graphicx}
\usepackage{textcomp}
\usepackage{xcolor}
\def\BibTeX{{\rm B\kern-.05em{\sc i\kern-.025em b}\kern-.08em
    T\kern-.1667em\lower.7ex\hbox{E}\kern-.125emX}}

\usepackage{bbm}
\usepackage{mathtools}
\usepackage{booktabs} 
\usepackage{multirow}

\newtheorem{definition}{Definition}
\newtheorem{corollary}{Corollary}
\newtheorem{lemma}{Lemma}
\newtheorem{assumption}{Assumption}
\newtheorem{theorem}{Theorem}
\newtheorem{proof}{Proof}
\newtheorem{remark}{Remark}
\newtheorem{proposition}{Proposition}

\usepackage{algorithm}
\usepackage{algpseudocode}

\newcommand{\calC}{{\cal C}}
\newcommand{\calD}{{\cal D}}

\newcommand{\calF}{{\cal F}}
\newcommand{\calG}{{\cal G}}

\newcommand{\calO}{{\cal O}}
\newcommand{\calP}{{\cal P}}

\newcommand{\calR}{{\cal R}}

\newcommand{\calT}{{\cal T}}

\newcommand{\calV}{{\cal V}}

\newcommand{\calX}{{\cal X}}

\newcommand{\bfb}{\mathbf{b}}

\newcommand{\bfh}{\mathbf{h}}

\newcommand{\bfx}{\mathbf{x}}

\newcommand{\bfalpha}{\boldsymbol{\alpha}}

\newcommand{\bftau}{\boldsymbol{\tau}}

\newcommand{\bbR}{\mathbb{R}}

\newcommand{\olive}[1]{\textcolor{olive}{#1}}
\newcommand{\red}[1] {\textcolor{red}{#1}}

\title{Draining Fictitious Knots: Restoring Distance-Awareness Guarantees for High-Dimensional Spline Networks
\thanks{This work is supported by the National Science Foundation under Grant No. 2218063 and the Gleason Endowment at RIT.}
}

\author{\IEEEauthorblockN{Masoud Ataei}
\IEEEauthorblockA{\textit{Electrical and Computer Engg.} \\
\textit{University of Maine}\\
Orono, ME, USA \\
masoud.ataei@maine.edu}
\and
\IEEEauthorblockN{Mohammad Javad Khojasteh}
\IEEEauthorblockA{\textit{Electrical and Microelectronic Engg.} \\
\textit{Rochester Institute of Technology}\\
Rochester, NY, USA \\
mjkeme@rit.edu}
\and
\IEEEauthorblockN{Vikas Dhiman}
\IEEEauthorblockA{\textit{Electrical and Computer Engg.} \\
\textit{University of Maine}\\
Orono, ME, USA \\
vikas.dhiman@maine.edu}
}
\begin{document}

\maketitle
\begin{abstract}
    Kolmogorov-Arnold Networks (KANs) with spline activations have recently shown promise for interpretable function approximation.  Distance-Aware Error for Kolmogorov Networks (DAREK) introduces a computationally efficient bottom-up approach to uncertainty quantification by equipping KANs with distance-aware error bounds; yet, in high-dimensional settings, the theoretical guarantees can be weakened by the emergence of fictitious knots.
    Inspired by the Kolmogorov-Arnold representation theorem, DAREK adopts a componentwise formulation in which each input dimension is treated separately; as a result, induced knot locations may appear in the combined input space without corresponding to actual training data.
    These fictitious knots mislead the DAREK uncertainty estimator into reporting low uncertainty far from any real observation, violating the distance-awareness guarantee. 
    We identify this failure mode precisely, characterize its geometric structure, and propose a drainage uncertainty mechanism that restores distance-awareness by constructing a monotonically decreasing uncertainty path from any fictitious knot region toward the nearest real knot. 
    The proposed drainage method provides a practical heuristic correction that mitigates the fictitious-knot failure mode while restoring theoretical distance-awareness in high-dimensional settings.
    Experiments on a 2D synthetic benchmark and a 100-dimensional face dataset show that drainage raises sampled
    distance-awareness (SDA) from ~85\% to 98-99\%, matching Gaussian processes at lower computational cost.
\end{abstract}


\section{Introduction}

Splines are smooth piecewise polynomials that have long been used for function approximation in control and signal processing~\cite{unser1999splines,egerstedt2009control}.
They have also been studied as adaptive activation functions in neural networks~\cite{guarnieri1999multilayer,igelnik2003kolmogorov,bohra2020learning}. Their recent use in Kolmogorov-Arnold Networks (KANs)~\cite{liu2025kan} has renewed interest in spline-based neural architectures due to their expressivity and interpretability. Beyond these properties, distance-aware uncertainty is crucial in safety-critical applications such as autonomous navigation, where learned models should behave conservatively away from the training distribution~\cite{ataei2025darek,liu2020simple}.
Fig.~\ref{fig:toy-example} shows a toy example where a standard classifier is uncertain only near its decision boundary ---whether
linear (left) or closed and nonlinear (middle)--- and remains confident (low uncertainty) elsewhere, including far from the \emph{training} data, causing a distant test point (star) to be labeled confidently.
A distance-aware model (right) instead reports high uncertainty away from the training points, correctly flagging the same distant test point as unreliable.

\begin{figure*}[t]
  \centering  \includegraphics[width=0.72\textwidth]{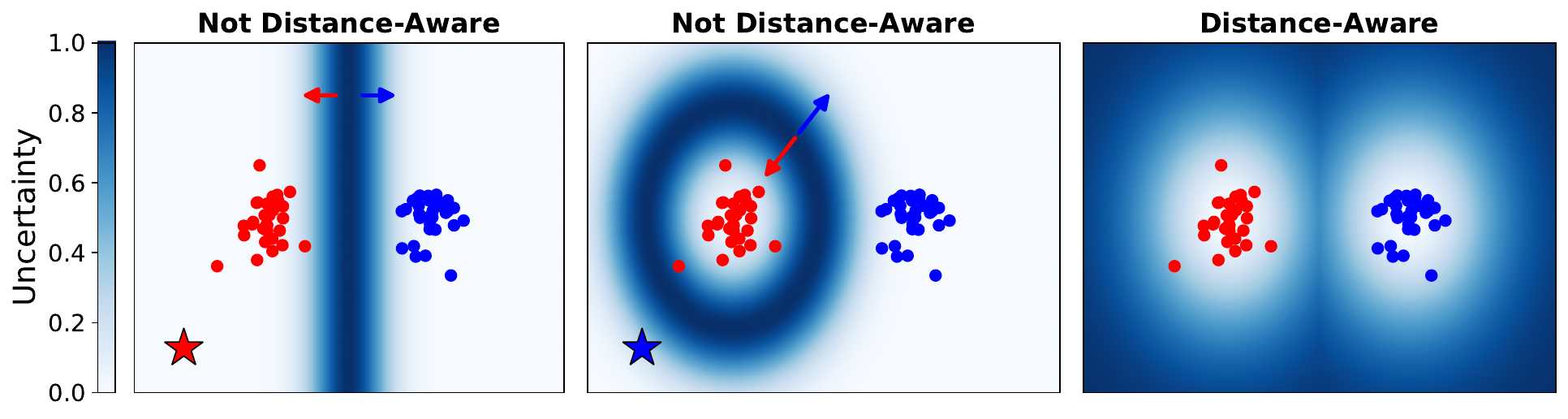}
  \caption{Toy binary classification.
  \textbf{Left)} and \textbf{middle)} Non-distance-aware models are confident (low uncertainty) far from the data. Models are uncertain only near the decision boundary. A linear separator (left) and a closed nonlinear one (middle). 
  Consequently, the distant query (star) lands in a confident region despite being far from training data.
  \textbf{Right)}, a distance-aware model raises uncertainty away from the training clusters.
  }
  \label{fig:toy-example}
  \vspace{-3mm}
\end{figure*}

Distance-aware uncertainty has traditionally been provided by Gaussian processes (GPs)~\cite{rasmussen2006gaussian} with kernels such as radial basis function (RBF), but their computational complexity grows cubically in the number of training samples.
Deep ensembles~\cite{lakshminarayanan2017simple} and Monte Carlo dropout (MC-D)~\cite{gal2016dropout} are cheaper, stochastic uncertainty estimators, but are not reliably distance-aware far from training data~\cite{liu2020simple}.  
Deterministic single-forward-pass methods such as SNGP~\cite{liu2020simple} and DUE~\cite{van2021feature} recover distance-awareness by pairing a distance-preserving feature extractor, enforced through spectral normalization or a bi-Lipschitz constraint, with a Gaussian-process output layer. However, these methods require architectural modification and end-to-end training, and their distance-awareness holds with respect to distances measured in the \emph{learned feature space}; DAREK, by contrast, defines distance-awareness directly in the input space.  
Verification methods such as CROWN~\cite{zhang2018efficient} instead bound the output of a fixed neural network over a prescribed input perturbation set using convex relaxations; they address robustness of the learned model, whereas DAREK bounds approximation error relative to the unknown target as a function of distance from observed data.
Recently, DAREK~\cite{ataei2025darek} introduced a deterministic worst-case uncertainty framework for spline neural networks (SNNs), providing an efficient, sampling-free alternative that bounds each spline's error and propagates the resulting uncertainty through the network composition.
Moreover, K-DAREK extends DAREK by applying distance-aware error bounds to Kurkova-Kolmogorov networks~\cite{kuurkova1992kolmogorov,toscano2025kkans}, combining MLP components with spline-based structures to improve efficiency and stability while leveraging their expressive power~\cite{ataei2025kdarek}.

Inspired by the Kolmogorov-Arnold representation theorem~\cite{schmidt2021kolmogorov,morris2021hilbert}, DAREK formulates distance-awareness componentwise, treating each input coordinate through separate univariate spline bounds. 
Although each component-level bound is distance-aware, its bottom-up propagation through summation and composition need not remain distance-aware in the full input space. This creates a geometric failure mode: coordinate-wise knots induce a grid of apparent knot locations, most of which are not actual training samples. We call these artificial locations \emph{fictitious knots}. Because the propagated uncertainty can vanish at fictitious knots just as it does at real knots, DAREK may assign low uncertainty to points that are far from all training observations, thereby weakening its high-dimensional distance-awareness guarantee.

In this paper, we characterize the geometry of fictitious knots and introduce a drainage-based uncertainty mechanism to restore distance-awareness by linking fictitious-knot regions to nearby real training knots.
Specifically, we 
\begin{itemize}
    \item identify and formalize the \textit{fictitious knot} problem in high-dimensional KANs, showing how componentwise spline processing can break joint distance-awareness;
    \item propose drainage uncertainty, which restores a monotone uncertainty path from fictitious knot regions to the nearest real knot;
    \item bound the drainage-region volume and characterize its dependence on drainage thickness and input dimension;
    \item validate the estimator on a 2D synthetic benchmark and a 100-dimensional task, raising sampled distance-awareness from roughly 85\% to 98--99\% and matching a GP at lower computational cost; and
    \item release the implementation and experiment code~\footnote{\url{https://github.com/Masoud-Ataei/Fictitious-Knots}}.    
\end{itemize}

Drainage provides a practical heuristic fix for fictitious knots, restoring theoretical distance-awareness while leaving the development of joint high-dimensional uncertainty bounds beyond componentwise methods for future work.

\section{Background}
In this section, we review the distance-awareness condition and the metric used to evaluate it, introduce spline neural networks, and summarize the worst-case error bounds underlying DAREK~\cite{ataei2026darek}.
{We use the notation $m$ for the number of knots, $k$ for the order of the spline, and $n$ for the input dimension.}
Let $f : [a, b] \to \bbR$ be a scalar function, $\calT = \{\tau_1, \dots, \tau_m\}$ an ordered set of $m$ distinct knots, and $\calD_f (\calT) = \{(\tau_i, f(\tau_i))\}_{i=1}^m$ the values observed at those knots.
Let the input space $\calX \subset \bbR^n$ and the finite set of training inputs $\calX_{\calD} \subset \calX$.
Knots are selected from the training inputs, so $\calT \subset \calX_{\calD}$.
A piecewise polynomial $\hat f$ of order $k$ over knots $\calT$ is a continuous function whose restriction to each interval $[\tau_j,\tau_{j+1})$ is a polynomial of order-$k$; we write {$\calP_{k,j}[\calD_{\hat f}(\calT)]$} for the $j$-th piece. 
We denote the $k$th derivative of a function $f$ by $f^{(k)}$.

\subsection{Distance-awareness}
A distance-aware uncertainty estimator is a model that reports lower confidence for query points that lie farther from the training data, as defined mathematically below. 
\begin{definition}[Input distance-awareness]
    \label{def:distance_awareness}
    Let $\hat{y} = \hat{f}(\bfx)$ approximate a target $y = f(\bfx)$ from inputs $x \in \calX_\calD$, and let $u_{\hat{f}}(\bfx) : \calX \to \bbR^+$
    be an uncertainty estimator that bounds the true values in the interval $y \in [\hat{f}(\bfx)-u_{\hat{f}}(\bfx), \hat{f}(\bfx)+u_{\hat{f}}(\bfx)]$
    for every $\bfx \in \calX$.
    Then $u_{\hat{f}}(\bfx)$ is \textbf{distance-aware} if it increases monotonically with the test point's distance from the training data under some distance function $d(\bfx, \calX_\calD)$.
\end{definition}

We use a specific distance function, which measures distance to the single closest training point (or knot) $\tau^*$ with a set distance $d(., \calX_\calD)$ defined as
\begin{align}
    d(\bfx, \calX_\calD) = \min_{\tau \in \calT} d_u(\bfx, \tau) = d_u(\bfx, \tau^*).
\end{align}
We call $d_u(\bfx, \tau^*)$ the \emph{inducing distance} of the uncertainty estimator $u_{\hat{f}}$.
It need not be Euclidean; it can be a geodesic distance on the data manifold. 
When both the uncertainty estimator and the inducing distance are differentiable, distance-awareness can be written as
\begin{align}
    [\nabla_\bfx u_{\hat{f}}(\bfx)]^\top \nabla_\bfx d_u(\bfx, \tau^*) \ge 0 \quad \forall \bfx \in \calX,
    \label{eq:dist-aware-cond}
\end{align}
meaning that uncertainty increases with increasing distance.

Checking~\eqref{eq:dist-aware-cond} at every point over $\calX$ is intractable, so we estimate it from samples, fixing the inducing distance to be Euclidean for comparability across estimators, as follows.

\begin{definition}[Sampled Distance-Awareness (SDA)]
    \label{def:sampled-distance-awareness}
    For a test point $\bfx_t$ drawn uniformly from $\calX_{\text{test}}$, {SDA is defined as}
    \begin{align}
        \label{eq:probability_distance_aware}
        \text{SDA} = \frac{1}{N}\sum_{\bfx_t \sim \calX_{\text{test}}}
        \mathbbm{1}
        \left[
        \nabla_{\bfx} u_{\hat{f}}(\bfx_t)^\top (\bfx_t - \bftau^*) \ge 0 \right],
    \end{align}
\end{definition}%
where $\mathbbm{1}[.]$ is the indicator function.
SDA measures the fraction of sampled directions along which uncertainty correctly rises away from the nearest knot.

\subsection{Error of a spline network}
Here we review the spline network we use (KANs) and the worst-case error bounds of the DAREK approach that our correction builds on.

\textbf{Spline:}
A spline of order $k$ is a piecewise polynomial of the same order that is $\calC^{(k-1)}$ times continuous.

\textbf{Spline networks:}
In KAN, the scalar weights of an MLP are replaced with learnable univariate spline activations. 
A two-layer KAN mapping $\bbR^{n_1}\to\bbR$ is $\mathrm{KAN}_2(\bfx)=\sum_{i=1}^{n_2}\phi_{2,1,i} \big(\sum_{j=1}^{n_1}\phi_{1,i,j}(x_j)\big)$, where each $\phi_{l,i,j}$ is a $k$th-order spline that connects input $j$ to output $i$ in layer $l$.
Each spline in the KAN implementation is constructed as $\phi_{l,i,j}(x_j)=\bfalpha_{l,i,j}^\top \bfb_{k,\tau_{l,j}}(x_j)$, where $\bfb_{k,\tau_{l,j}}$ are B-spline bases.
$l$th layer of KAN can be written as $\bfh_l(\bfx) = (\sum_{j=1}^{n_1}\phi_{l,i,j}(x_j))_{i=1}^{n_l}$ and an $L$-layer KAN is the composition
\begin{equation}\label{eq:kan_comp}
\mathrm{KAN}_L(\bfx)=\bfh_L\circ\cdots\circ \bfh_1(\bfx),
\qquad \bfh_l:\bbR^{n_l}\to\bbR^{n_{l+1}}.
\end{equation}
Choosing the knots as a subset of the training inputs, $\calT \subset \calX_{\calD}$, as proposed in DAREK~\cite{ataei2025darek}, fixes the basis functions by the data, so only the coefficients $\alpha_{l,i,j}$ are learned, making the KAN a \textit{semi-parametric} model.
Here $n_l$ is the dimension of layer $l$ and $n_1=n$ is the input dimension.

\textbf{DAREK}~\cite{ataei2025darek} computes a worst-case error bound for an SNN at a test point, using $m$ selected training inputs as the knots of the spline network. We review the DAREK results needed here; complete derivations and proofs are available in~\cite{ataei2026darek}.
DAREK additionally makes the following Lipschitz-continuity assumption:

\begin{assumption}[$k$th-order Lipschitz continuity]\label{ass:lip}
For a $(k-1)$-times differentiable $f:[a,b]\to\bbR$, the $k$th-order
Lipschitz constant $L_f^k$ satisfies
$|f^{(k-1)}(x)-f^{(k-1)}(y)|\le L_f^k\,d(x,y)$ for all $x\neq y$, where $f^{(k)}$ is the $k$th derivative of function $f$.
\end{assumption}

\begin{theorem}[Newton interpolation error bound~\cite{ataei2025darek}]
\label{thm:poly-interp-bound}
Let $f\in C^{k+1}$ on $[a,b]$ be $(k{+}1)$th-order Lipschitz with constant
$L_f^{k+1}$. The $k$th-order Newton piecewise-polynomial fit
$\calP_{k,j}[\mathcal{D}_f(\calT)]$ through the knots satisfies, for
$x\in[\tau_j,\tau_{j+1})$,
\begin{align}
    |f(x)-\calP_{k,j}[\mathcal{D}_f](x)| &\le
    \frac{L_f^{k+1}}{(k+1)!}\Big|\textstyle\prod_{i=1}^{k+1}(x-\tau_i^{(j)})\Big| \notag\\
    & := \Bar{u}_f(x;\calT),
    \label{eq:int-error}
\end{align}
which we denote $\bar u_f(x;\calT)$ as interpolation error bound and $\tau_i^{(j)}$ are the $k{+}1$ knots close by $x$.
\end{theorem}
The proof is provided in~\cite{ataei2026darek}. We restate the result here to establish the notation used throughout the remainder of this paper.

\begin{proposition} [Distance-awareness of the interpolation-error bound~\cite{ataei2026darek}]
\label{prop:1d-da}
Let $f:[a,b]\to\bbR$ be the scalar function with knots $\calT=\{\tau_1,\dots,\tau_m\}$.
On each interval $[\tau_j,\tau_{j+1})$, the bound $\bar u_f(\cdot;\calT)$ of
\eqref{eq:int-error} vanishes at both knots, is strictly positive in the interior, and has a unique maximum $x_j^\star\in(\tau_j,\tau_{j+1})$.
It is strictly increasing on $(\tau_j,x_j^\star)$ and
strictly decreasing on $(x_j^\star,\tau_{j+1})$.
\end{proposition}

\textbf{Error bound with linear fit (EBL):}
In practice, the trained network does not pass exactly through the knots, so there is a residual $e^f_j(x):=f(x)-\hat f_{[j]}(x)$ whose values at the knots, $\calD_{e^f_j}(\tau_{1:m})$, are known. 
We account for this residual by adding a piecewise-linear interpolation of the absolute knot errors, $|\calP_{1,j}[\calD_{e_j^f}](x)|=[\tau_{j}, \tau_{j+1}]|e_j^f| (x-\tau_j) + |e_j^f|$, where the slope is the divided difference $[\tau_{j}, \tau_{j+1}]|e_j^f|= ( |e_{j+1}^f| - |e_j^f| ) / (\tau_{j+1} - \tau_j)$ and refer to this approach as \textit{error bound with linear fit} (EBL).
\begin{align}
    \label{eq:error-at-knots-ebl}
    u_f(x; \bftau_{1:m}) \coloneqq &\bar{u}_f(x; \bftau_{1:m}) + |\calP_{1,j} [\calD_{e_{j}^f}](x)| \quad \text{(EBL)}.
\end{align}

\section{Problem Statement}
\label{sec:problem}
In this section, we formalize the failure mode that motivates the paper.
We first define real and fictitious knots (Sec.~\ref{sec:real-vs-fictitious}), then demonstrate that the bottom-up DAREK bound vanishes at every fictitious knot, and, as a consequence, need not remain distance-aware once the input dimension exceeds one (Sec.~\ref{sec:violation-fic-knots}).

\subsection{Real and fictitious knots}
\label{sec:real-vs-fictitious}
In KAN-based architectures, a separate univariate spline is applied to each input coordinate.
Consider $\calT$ with $m$ selected knots from $N$ training samples ($m \le N$), used as the knots of a single layer's spline. 
The spline acting on coordinate $j$ therefore takes its knots from the projection of the selected points onto that axis, $\calT_j=\{\bftau_{1,j},\dots,\bftau_{m,j}\} \subset \bbR$, for $j \in \{1, \dots, n\}$.
Because the coordinates are processed independently, the first-layer interpolation-error term $u(x)=\sum_{j=1}^{n}\bar u_j(x_j)$ depends on $x$ only through the $n$ separate proximities of each $x_j$ to the axis knots, never through the joint position of $x$. 
Since each $\bar u_j\ge 0$ vanishes exactly at its knots (Theorem~\ref{thm:poly-interp-bound}), $u$ vanishes precisely on the grid $\calG= \calT_1\times \calT_2\times\cdots\times \calT_n$, where its cardinality is $|\mathcal{G}|=m^{n}$.

For instance, consider $n=2$ and two selected knots $\bftau_1=(-1,0)$ and $\bftau_2=(1,2)$. 
Their per-axis projections are $\calT_1=\{-1,1\}$ and $\calT_2=\{0,2\}$, so the grid is $\calG=\calT_1\times \calT_2=\{(-1,0),\,(-1,2),\,(1,0),\,(1,2)\}$, with $|\calG|=2^{2}=4$ vertices, even though only the two points $\bftau_1,\bftau_2$ were ever selected. 
Of these four grid vertices, $\calR=\{(-1,0),(1,2)\}$ are \emph{real knots}, while $\calF=\{(-1,2),(1,0)\}$ are \emph{fictitious knots}, arising only from recombining coordinates across axes. Fig.~\ref{fig:contour} shows a $5 \times 5$ grid produced by $m=5$ knots in $n=2$ dimensions.

\begin{definition}[Real and fictitious knots]
\label{def:real-vs-fictitious}
A grid vertex $g\in\mathcal{G}$ is a \textit{real knot} if it is one of the selected knot points, $g\in\calR:=\calT$, and a \textit{fictitious knot} otherwise, $g\in\calF:=\calG\setminus\calR$. 
\end{definition}

Real knots are the points the knot selection intended to mark; fictitious knots are the spurious coincidences created by recombining the per-axis projections.
Each selected point contributes its projection to every $\calT_j$, so $\calT\subseteq\calG$ and $|\calR|=m$, while $|\calG|=m^{n}$. 
The fraction of grid vertices that are real is therefore $\frac{|\calR|}{|\calG|}=\frac{m}{m^{n}}\xrightarrow[n\to\infty]{}0\,\,\,(m\ge 2)$, so in high dimensions the apparent knot grid is almost entirely fictitious.

The interpolation-error term cannot distinguish the two cases: $\bar u_j$ vanishes at every axis knot, so $u(g)=0$ for every $g\in\calG$, real or fictitious. 
The distance to the data $d(x,\calX_D)=\min_{\tau\in\calR}\lVert x-\tau\rVert$ vanishes only on $\calR$.
A fictitious knot can lie arbitrarily far from every training input while still receiving zero interpolation error. 
Consequently, the theoretical distance-awareness guarantee can weaken in high-dimensional settings, where the bound may decrease to near-zero uncertainty at fictitious knots even though such locations can be far from the training data.
We call this the \textbf{fictitious-knot problem}. 
Note that this is a theoretical weakening of the guarantee, not necessarily a large practical effect: our experiment with DAREK still observes SDA around 84-87\% (Sec.~\ref{sec:experiments}, Table~\ref{tbl:drainage-uncertainty}) even without correction; the drainage method closes this remaining gap and formally restores the guarantee.

Equivalently, the per-coordinate ``nearest knot'' of Theorem~\ref{thm:poly-interp-bound} is not a point but the hyperplane $\{x:x_j=\tau\}$; the $n$ such hyperplanes intersect in the grid $\calG$, and the bottom-up bound treats a query as close to the data whenever it is near one hyperplane in \textit{each} coordinate at once, irrespective of its joint location.

\begin{figure}[t]
  \centering
  \includegraphics[width=\columnwidth]{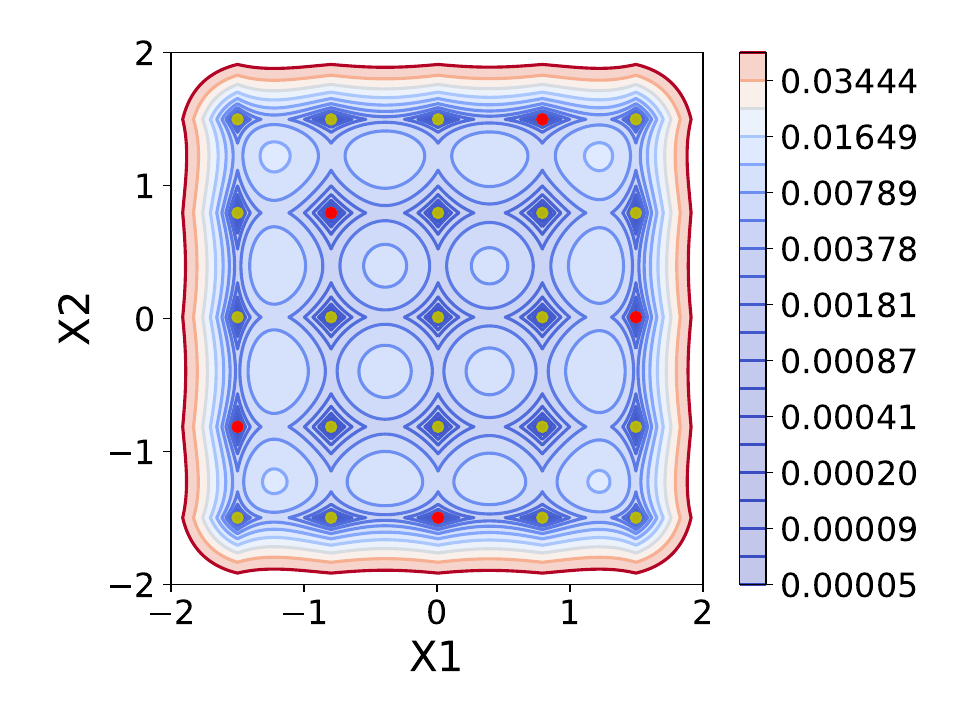}
  \vspace{-7mm}
  \caption{Interpolation-error bound $\bar u$ of \eqref{eq:int-error}. 
  \red{Red circles ($\bullet$)} are real knots $\calR$ and \olive{olive circles ($\bullet$)} are fictitious knots $\calF$.
  Near-zero uncertainty appears at every vertex of the $5\times5$  grid, not just at the five real knots: $\bar u$ vanishes across all of $\calG$, so fictitious vertices receive near-zero uncertainty.} 
  \label{fig:contour}
\end{figure}

\subsection{The violation at fictitious knots}
\label{sec:violation-fic-knots}
The collapse identified in Sec.~\ref{sec:real-vs-fictitious} ($u\equiv 0$ on the whole grid $\calG$) is not merely a loss of tightness; 
it creates regions where uncertainty can decrease while distance from the data increases at fictitious knots, as stated in Lemma~\ref{lem:fictitious-violation}. 

\begin{lemma}[Fictitious-knot violation]
\label{lem:fictitious-violation}
Let $u(x)=\sum_{j=1}^{n}\bar u_j(x_j)$ be the first-layer interpolation-error bound, and let $g\in\calF$ be a fictitious knot whose nearest real knot is $\tau^\ast$, so $d(g,\calX_D)=\lVert g-\tau^\ast\rVert>0$. 
Then on the approach to $g$ along the segment from $\tau^\ast$, the uncertainty $u$ strictly decreases while $d(\cdot,\calX_D)$ strictly increases, so $[\nabla_x u(x)]^{\top}\nabla_x d(x,\mathcal{X}_D)<0$ on a neighborhood of $g$, violating~\eqref{eq:dist-aware-cond}. 
Hence, the DAREK bound is distance-aware component-wise, but the propagated bound in the full input space for $n>1$ is not.
\end{lemma}

\begin{proof}
Since $g$ has nearest real knot $\tau^\ast$ and Voronoi cells are convex polytopes, the segment $\{x(t)=\tau^\ast+t(g-\tau^\ast):t\in[0,1]\}$ lies in $\calV(\tau^\ast)$, where $d(x(t),\calX_D)=\lVert x(t)-\tau^\ast\rVert=t\,\lVert g-\tau^\ast\rVert$ increases in $t$. 
By Definition~\ref{def:real-vs-fictitious}, both endpoints lie on $\calG$, so $u(\tau^\ast)=u(g)=0$, while $u>0$ on the open segment because the product from~\eqref{eq:int-error} is strictly positive away from knots. Thus $g$ is an isolated minimizer of $u$ and $u$ is strictly decreasing as $x(t)\to g$. Therefore $\nabla_x u$ points back toward $\tau^\ast$ (the direction of increasing $u$), whereas $\nabla_x d(\cdot,\calX_D)=(x-\tau^\ast)/\lVert x-\tau^\ast\rVert$ points away from $\tau^\ast$; the two gradients are opposed and their inner product is negative.
\end{proof}

Composition across layers is expected to propagate rather than repair the violation, since each layer adds the previous layers' summed error; we do not analyze deeper-layer composition separately (see Limitations). 
The violation stays confined to the geometrically small fictitious-knot regions characterized in Sec.~\ref{sec:real-vs-fictitious}; DAREK's overall SDA remains above chance because those regions are a small fraction of the input space. We correct the estimate directly in Sec.~\ref{sec:drainage-estimator}.

\section{Methodology}
With the problem defined in Sec.~\ref{sec:problem}, we now introduce a drainage-based correction, prove that it restores distance-awareness under an explicit condition on its gain (Theorem~\ref{thm:restored-da}), and bound the fraction of the input space the correction modifies (Corollary~\ref{col:volume-ratio}).

\subsection{Drainage Estimator}
\label{sec:drainage-estimator}
We reshape the uncertainty so that a monotone path to the data is restored, without retraining the network or altering its knots. 
This is achieved by introducing a linear ``drainage'' path, which smoothly redirects uncertainty toward the nearest real knot.
Let $\tau^\ast(x)$ be the nearest real knot to $x$, $\tau^{+}$ the nearest fictitious knot to $\tau^\ast$, and $\Delta=\lVert\tau^\ast-\tau^{+}\rVert$ the spacing from a real knot to its nearest grid neighbor; we first define the drainage term abstractly.

\begin{definition}[Drainage uncertainty]
    A nonnegative function $u_d:\calX \to\bbR^+$ is a \textbf{drainage uncertainty} for the real knots $\calR$ if it is expressed through the inducing distance $d(x,\calX_D)$ and satisfies 
    \begin{align}
    \text{(D1)}\, &u_d(\tau)=0\ \ \forall\tau\in\mathcal{R}, \\ \text{(D2)}\, &u_d\ \text{is strictly increasing in }d(x,\mathcal{X}_D).
    \end{align}
\end{definition}
Property (D2) makes $u_d$ distance-aware by construction: writing $u_d(x)=\psi\big(d(x,\mathcal{X}_D)\big)$ with $\psi'>0$ gives $[\nabla_x u_d]^{\top}\nabla_x d = \psi'\lVert\nabla_x d\rVert^2\ge 0$. 
Two natural choices are the linear ramp and a saturating (RBF-type) form,
$u_d^{\mathrm{lin}}(x)=\eta\,\frac{d(x,\mathcal{X}_D)}{\Delta}$, $u_d^{\mathrm{rbf}}(x)=\eta\Big(1-e^{-d(x,\mathcal{X}_D)^2/2\ell^{2}}\Big)$, where $\Delta=\lVert\tau^\ast-\tau^{+}\rVert$ is the spacing from a real knot to its nearest grid neighbor and $\eta$,$\ell>0$. 
The linear form rises at a constant rate set by a single gain $\eta$; the saturating form connects to the kernels of distance-aware Gaussian processes, trading the gain for a length-scale $\ell$.

\begin{definition}[Drainage-corrected estimator] 
Given a drainage uncertainty $u_d$, the drainage-corrected estimator $u_{\text{da}}$ is
\begin{align}
    \label{eq:combined-error}
    u_{\text{da}}(x) =
    \begin{cases}
        u(x), & |x_j-\tau^\ast_j(x)|<\epsilon\ \ \forall j,\\[2pt]
        \max\{\,u(x),\,u_d(x)\,\}, & \text{otherwise.}
    \end{cases}
\end{align}
\end{definition}
The first case $|x_j-\tau^\ast_j(x)|<\epsilon\ \ \forall j$ corresponds to the genuine box around the nearest real knot. Inside this box, the principled DAREK bound is used unchanged; elsewhere, the estimate may be lifted by drainage. 
By construction, $u_{\mathrm{da}}\ge u$ everywhere, so the worst-case enclosure $y\in[\hat f\pm u_{\mathrm{da}}]$ is preserved (drainage never reports less than DAREK), and $u_d\to 0$ as $x\to\tau^\ast$, so the lift vanishes at the data and the correction is not over-conservative near training points.

\begin{remark}
The combination uses $\max$ rather than replacing $u$: a query may satisfy the partial-proximity condition yet have small $u_d$, and replacing $u$ there could report less than the original DAREK  bound.
Taking the maximum keeps the larger of the DAREK bound and the distance-aware ramp.
\end{remark}

In the remainder, we instantiate $u_d=u_d^{\mathrm{lin}}$, chosen for its single interpretable gain.
Theorem~\ref{thm:restored-da} below gives an explicit sufficient threshold condition on this gain for restoring distance-awareness. 

\subsection{Restored distance-awareness}
Drainage restores the distance-awareness condition~\eqref{eq:dist-aware-cond} when its gain is sufficiently large.
To state this condition, within a Voronoi cell $\calV(\tau^\ast)$, let $U_{\max}=\sup_{y\in\mathcal{V}(\tau^\ast)}u(y)$ denote the largest value of the DAREK bound in that cell.

\begin{theorem}[Restored distance-awareness]
\label{thm:restored-da}
Suppose
\begin{align}
&\text{(C1)}\quad \epsilon<\tfrac12\min_j(\text{adjacent-knot gap on axis }j),
\qquad\\
&\text{(C2)}\quad \eta\ge \frac{\Delta\,U_{\max}}{\epsilon}.
\end{align}

Then for every $x\in\mathcal{V}(\tau^\ast)$ the segment from $x$ to $\tau^\ast$ is a path along which $u_{\mathrm{da}}$ is monotonically non-increasing, ending at $u_{\mathrm{da}}(\tau^\ast)=u(\tau^\ast)$. 
Consequently $u_{\mathrm{da}}$ satisfies the distance-awareness condition~\eqref{eq:dist-aware-cond} on $\calV(\tau^\ast)$, and hence on all of $\calX$.
\end{theorem}
\begin{proof}
Parametrize the segment by $s=\lVert y-\tau^\ast\rVert$, so $d(y,\mathcal{X}_D)=s$ and $u_d(y)=(\eta/\Delta)s$ is linear and increasing. 
By (C1) the only grid knot in the $\epsilon$-box is $\tau^\ast$, so for $s<\epsilon$ the genuine box gives $u_{\mathrm{da}}=u$, which rises from $u(\tau^\ast)$ as a single isolated knot (non-decreasing in $s$). 
Leaving the box forces some coordinate to differ from $\tau^\ast$ by at least $\epsilon$, so $s\ge\epsilon$; then by (C2), $u_d(y)=(\eta/\Delta)s\ge(\eta/\Delta)\epsilon\ge U_{\max}\ge u(y)$, hence $u_{\mathrm{da}}=\max\{u,u_d\}=u_d=(\eta/\Delta)s$, again increasing in $s$. 
Thus $u_{\mathrm{da}}$ is non-decreasing in $s$, i.e. non-increasing toward $\tau^\ast$; since $d(\cdot,\mathcal{X}_D)=s$, \eqref{eq:dist-aware-cond} holds. 
Repeating per cell gives the global claim. \end{proof}

\begin{remark}[The gain is conservative]
Condition (C2) bounds drainage against the worst case $U_{\max}$ over the whole cell and is sufficient, not necessary; in practice, much smaller gains restore distance-awareness, as the ablation of Sec.~\ref{sec:experiments} confirms. 
A single global $\Delta=\min_{\tau\in\calR} \|\tau - \tau^{+}(\tau)\|$, the minimum grid spacing over all real knots, may be used in place of the per-knot $\Delta$ to make $u_d$ continuous across Voronoi boundaries, at the cost of potentially more conservative drainage values in cells with larger per-knot spacing.
\end{remark}

\begin{remark}[Computation]
Drainage adds negligible cost to a DAREK forward pass. 
Per real knot, $\tau^{+}$ and $\Delta$ are precomputed once, and 
each query costs $\calO(n)$ beyond the nearest-knot lookup. The correction requires no additional model evaluations.
\end{remark}

\subsection{Volume of Drainage Region}
Although the drainage uncertainty restores distance awareness, the drainage term is not itself derived as a worst-case error bound and may increase the combined bound in the region where it applies. 
We therefore want the drainage region to be small relative to the rest of the input space.
In this section, we quantify this proportion as a function of the thickness $\epsilon$, the number of knots $m$, and the input dimension $n$.

\begin{corollary}[Ratio of Drainage to Input Volume]
\label{col:volume-ratio}
Let each input dimension span an interval of length $a$, with $m$ knots per dimension, drainage thickness $\epsilon$ around each knot, and non-overlapping neighborhoods of radius $\epsilon$ around the knots.
The drainage ratio per dimension becomes $r = 2m\epsilon/a \le 1$. Let the drainage region be the set of points for which at least one coordinate lies within $\epsilon$ of a knot on its corresponding axis; then the fraction of the input volume occupied by the drainage region is $\Delta S \;=\; 1-(1-r)^{n}$.
\end{corollary}
\begin{proof}
A point lies outside the drainage region if and only if, for every coordinate, it lies farther than $\epsilon$ from all $m$ knots on the corresponding axis. 
Along each coordinate, the knot neighborhoods occupy a total length of $2m\epsilon$, leaving a fraction $1-2m\epsilon/a = 1-r$ outside these neighborhoods, and the fraction lying outside the knot neighborhoods in every coordinate is $(1-r)^{n}$.
Therefore, the fraction inside the drainage region is $\Delta S \;=\;1-(1-r)^{n}$. 
\end{proof}

\begin{remark}[A thickness–dimension trade-off]
For any \textbf{fixed} relative thickness $r>0$, $\Delta S = 1-(1-r)^{n}\to 1$ as $n\to\infty$.
Thus, in high dimensions, the union of the knot slabs occupies nearly all of the input domain. 
This product-volume effect is distinct from the combinatorial growth of the $m^n$ fictitious knots. 
To keep the drainage region at a target volume fraction $r_0$, one may choose $r =1-(1-r_0)^{1/n}$, or, equivalently, $\epsilon =a/(2m) [1-(1-r_0)^{1/n}] = \calO\big(a/(m n)\big)$.
However, this reduction in thickness comes at the cost of a larger sufficient gain in (C2), since its lower bound is inversely proportional to $\epsilon$. 
Thus, the choice of $\epsilon$ represents a trade-off between the volume of the drainage region and the gain required to guarantee distance-awareness. 
\end{remark}

\section{Experiments}
\label{sec:experiments}
We evaluate drainage through four studies: a 2D synthetic benchmark that visualizes the fictitious-knot failure mode and its correction; a 100-dimensional face bounding-box task that tests the method in a high-dimensional setting; an ablation over the drainage thickness $\epsilon$, gain $\eta$, and knot count $m$; and a computational-cost comparison against DAREK, ensembles, and Gaussian processes. 
The uncertainty quality is reported as the sampled distance-awareness (SDA) of
Definition~\ref{def:sampled-distance-awareness}.

\subsection{2D synthetic benchmark}
\label{sec:exp-2d}
 
\textbf{Setup:}
The input domain is $\calX=[-2,2]^2$ and the target is the smooth multimodal function $f(x_1,x_2)=\sin(x_1)\cos(x_2)$. 
We draw $N=1000$ training points uniformly and select
$m=5$ of them as spline knots (real knots). 
Each of these knots contributes to each coordinate. Because DAREK analyzes the interpolation error separately along each coordinate, the five values on the first axis combine with the five values on the second axis, producing a grid of $5 \times 5 = 25$ apparent knots (Definition~\ref{def:real-vs-fictitious}). Five of these knots preserve the original coordinate pairings and are real knots; the remaining 20 combine coordinates from different training points and are fictitious knots. 
DK1 and DK2 are one-layer and two-layer DAREK models, and ENS1 and ENS2 are ensembles of 5 one-layer and two-layer KANs, respectively. 
All models use cubic ($k=3$) B-spline activations and grid size 5. 
The approximator is DK1 unless otherwise specified.
We use the linear ramp $u_d^{\mathrm{lin}}$ with gain $\eta=0.03$ and box half-width $\epsilon=0.05$ so that the conditions in Theorem~\ref{thm:restored-da} hold.
SDA is estimated from 2000 points
drawn uniformly on $\calX$.

\textbf{Uncertainty collapse:}
Fig.~\ref{fig:contour} plots the $\bar u$~\eqref{eq:int-error} with selected contour levels.
Concentric wells appear at \emph{every} vertex of the $5\times5$ grid $\calG$, not only at the five real knots $\calR$ (red).
$\bar u$ vanishes on all of
$\calG$, exactly as Definition~\ref{def:real-vs-fictitious} predicts. 
A query sitting at a fictitious vertex therefore receives near-zero uncertainty despite being far from any training point as formalized in Lemma~\ref{lem:fictitious-violation}.
 
\textbf{Draining uncertainty:} Fig.~\ref{fig:ubar_drainage} shows the three terms side by side. 
The DAREK panel (left) has low values across the whole interior grid, high only at the far corners. 
The drainage panel (middle), $u_d^{\mathrm{lin}}$, is zero at the real knots and rises with distance to the nearest one. 
The combined estimate (right), $u_{\mathrm{da}}=\max\{u,u_d\}$ of \eqref{eq:combined-error}, keeps the principled DAREK bound inside the $\epsilon$-box around each real knot and lifts the fictitious minima everywhere else, so the spurious wells disappear.
 
\textbf{Distance-awareness violations:}
Fig.~\ref{fig:violation} visualizes the distance-awareness violation regions of the interpolation error~\eqref{eq:int-error}, drainage error, and combined error~\eqref{eq:combined-error} for the two different sets of knots.
The small dots are test points and their color shows the uncertainty. 
The red circles indicate selected knots, and the highlights indicate regions where violations occur.
We evaluate the condition~\eqref{eq:dist-aware-cond} on each test point.
These regions concentrate at the fictitious knots, confirming that the failure is of geometric origin rather than a sampling artifact.

\textbf{Monotone path to real knots:}
Fig.~\ref{fig:vfield} overlays streamlines of $-\nabla u_{\mathrm{da}}$
on the combined uncertainty field. Every streamline terminates at a
real knot (the low-uncertainty wells), so from an arbitrary query
following decreasing uncertainty leads to actual training data, which is guaranteed by Theorem~\ref{thm:restored-da}.
 
\textbf{Compared with RBF:} Fig.~\ref{fig:rbf} normalizes the combined estimate and a RBF-kernel GP to $[0,1]$. Both are low at the real knots and rise away from them; drainage reproduces the GP's distance-aware bowl structure, differing mainly in the residual slab pattern left by the per-axis grid. 
 
\textbf{Quantitative result:}
On this benchmark, the SDA of the uncorrected DAREK variants is
$\approx$\,84--87\% (Table~\ref{tbl:drainage-uncertainty}, DK1/DK2), reflecting
the directions near fictitious knots where uncertainty wrongly
decreases outward. 
Adding drainage raises SDA to 99.5--99.8\%, matching the GP 100\% while the other baselines (deep ensembles ENS1/ENS2) stay near random chance~50\%.

\subsection{Face bounding box}
\label{sec:face-dataset}
We evaluate drainage on a high-dimensional face task. The task is to predict the bounding box (coordinates of the face) from an image on the Celebrity Attribute (CelebA) dataset~\cite{liu2015deep}. 
We use $8k$ training and $1k$ test images, and reduce each image to a $100$-dimensional input via PCA feature projection. 
All KANs use cubic ($k=3$) splines with 15 knots.

Because drainage only lifts the uncertainty estimate and never alters the network output, RMSE and IoU remain unchanged and are not the focus of this paper; the comparison focuses on SDA alone. 
As reported in Table~\ref{tbl:drainage-uncertainty} (100D row), DAREK attains
$95.7$--$96.3\%$ SDA, and drainage raises it to $98.4$--$99.9\%$, matching the GP ($99.55\%$) while the ensembles remain near random chance ($\approx 50\%$). 
Drainage adds only $\mathcal{O}(n)$ per query, so it preserves DAREK's order-of-magnitude speed advantage over the cubic-cost GP.
\begin{table}[!ht]
    \centering    
    \caption{Sampled Distance-awareness (SDA) for different models for the 2D and 100D datasets.}
    \label{tbl:drainage-uncertainty}
    \resizebox{0.45\textwidth}{!}{%
    \begin{tabular}{lccccccc} 
        \hline
        \textbf{Model} & {DK1} & {DK2} & ENS1 & ENS2 & GP & DK1 + drainage & DK2 + drainage\\   
        \hline
        2D      & 84.30 & 87.25 & 44.85 & 52.68 & 100.0 & 99.50 & 99.75 \\
        100D    & 96.27 & 95.67 & 52.09 & 50.15 & 99.55 & 99.85 & 98.36 \\
        \hline
    \end{tabular}}
\end{table}

\subsection{Ablation and Computational Cost}
\label{sec:ablation}
\begin{table}[t]
\centering
\caption{Ablation on the drainage parameters. The baseline DAREK
distance-awareness (SDA$_{\text{int.}}$~\eqref{eq:int-error}, no drainage) depends only on $m$;
SDA$_{\text{combined}}$ (\%) reports the full correction of Eq.~\eqref{eq:combined-error} for four
values of the slope $\eta$. The drainage volume ratio $\Delta S$ depends only
on $(m,\epsilon)$ and the test loss only on $m$. Bold marks the smallest
$\Delta S$ that attains SDA $=100\%$.}
\label{tab:drainage_ablation}
\begin{tabular}{cccccccc}
\toprule
 & & & & \multicolumn{4}{c}{SDA$_{\text{combined}}$ (\%) at $\eta=$} \\
\cmidrule(lr){5-8}
$m$ & SDA$_{\text{int.}}$ & $\epsilon$ & $\Delta S$ & $0.01$ & $0.03$ & $0.1$ & $0.3$ \\
\midrule
\multirow{4}{*}{5} & \multirow{4}{*}{89.0}
  & 0.025 & 0.121 & 98.50 & 99.80 & \textbf{100.0} & 100.0 \\
& & 0.05  & 0.234 & 98.80 & 99.90 & 100.0 & 100.0 \\
& & 0.10  & 0.438 & 99.10 & 99.90 & 100.0 & 100.0 \\
& & 0.20  & 0.750 & 99.10 & 99.80 & 99.90 & 99.90 \\
\midrule
\multirow{4}{*}{7} & \multirow{4}{*}{74.3}
  & 0.025 & 0.167 & 95.70 & 99.70 & \textbf{100.0} & 100.0 \\
& & 0.05  & 0.319 & 96.10 & 99.70 & 100.0 & 100.0 \\
& & 0.10  & 0.578 & 96.70 & 99.80 & 100.0 & 100.0 \\
& & 0.20  & 0.910 & 98.30 & 99.90 & 99.90 & 99.90 \\
\midrule
\multirow{4}{*}{9} & \multirow{4}{*}{77.0}
  & 0.025 & 0.212 & \textbf{100.0} & 100.0 & 100.0 & 100.0 \\
& & 0.05  & 0.399 & 100.0 & 100.0 & 100.0 & 100.0 \\
& & 0.10  & 0.698 & 100.0 & 100.0 & 100.0 & 100.0 \\
& & 0.20  & 0.990 & 100.0 & 100.0 & 100.0 & 100.0 \\
\bottomrule
\end{tabular}
\end{table}
In Table~\ref{tab:drainage_ablation}, we study how the drainage parameters, thickness $\epsilon$ and the slope $\eta$, and the knot count $m$ affect SDA, the drainage volume ratio ($\Delta S $), and the SDA of interpolation error in the 2D experiment.
Without drainage, the interpolation error achieves SDA values between 74.3\% and 89.0\%. Introducing the drainage term consistently increases SDA, reaching 100\% for all tested models ($m$=5, 7, 9) under appropriate parameter settings. 
The greatest improvement occurs for $m=7$, where SDA increases from 74.3\% without drainage to 100\% under appropriate drainage parameters.
The results further indicate that increasing $\eta$ improves SDA when violations are present, but the benefit quickly saturates in this experiment when $\eta=0.1$.
Moreover, at an adequate slope ($\eta\geq0.1$), the smallest tested drainage region ($\epsilon=0.025$) suffices to attain 100\% SDA for every model order, corresponding to drainage volume ratios of only 0.121, 0.167, and 0.212 for $m=5,7,9$, respectively. 
These findings suggest that the residual violations are confined to small regions that drainage can target without substantially enlarging the corrected volume.

\begin{figure}
    \centering
    \includegraphics[width=0.48\textwidth]{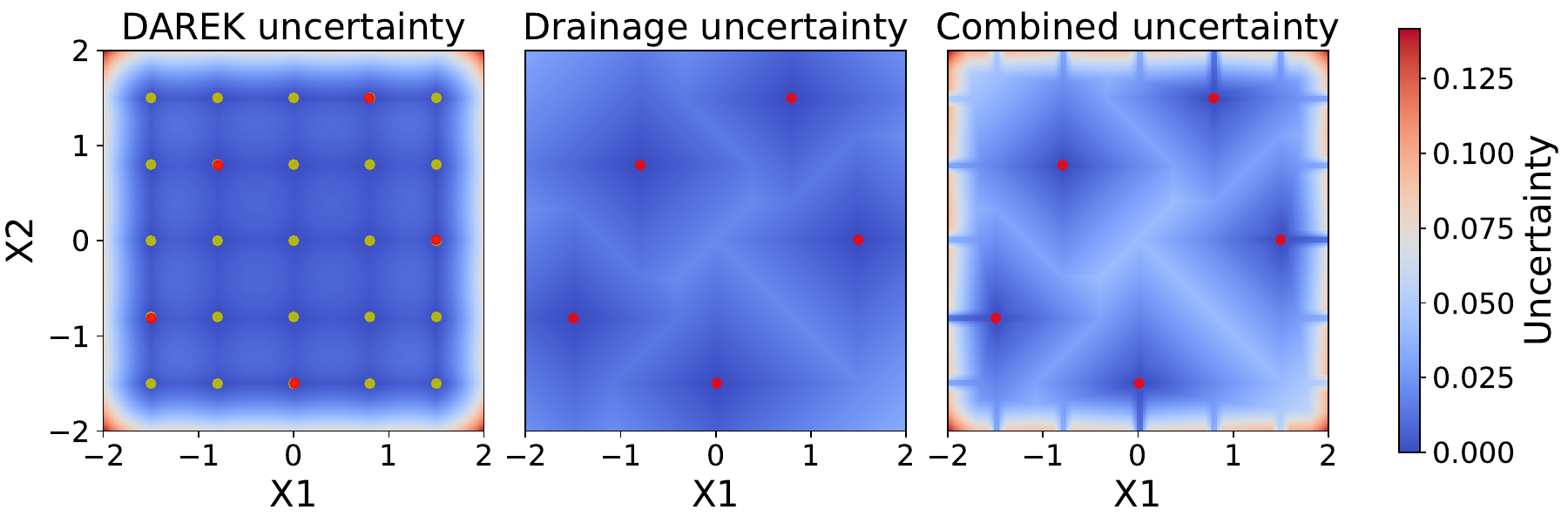}
    \includegraphics[width=0.48\textwidth]{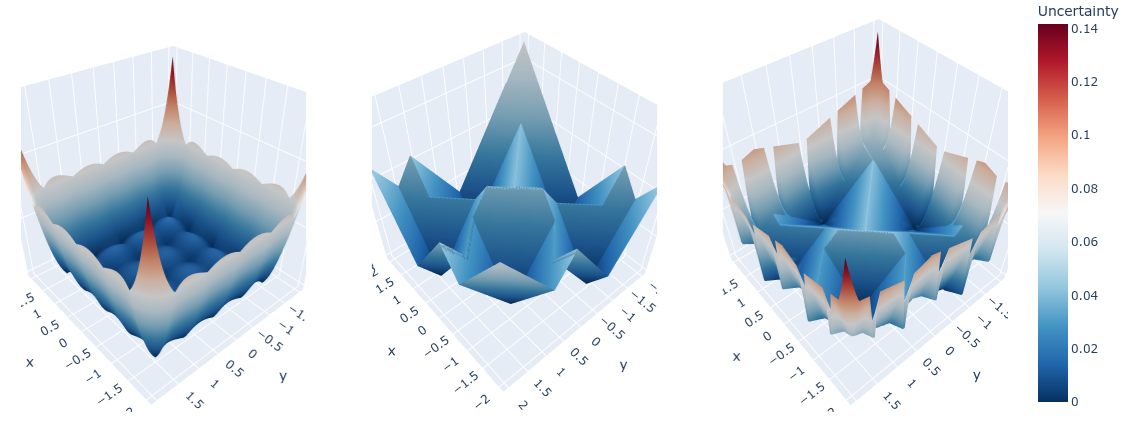}
    \caption{
    DAREK uncertainty.  \textbf{(left)} the original interpolation-error bound~\eqref{eq:int-error}; \textbf{(middle)} the linear drainage uncertainty; and \textbf{(right)} the combined estimate~\eqref{eq:combined-error}.   
    Red and olive markers denote real and fictitious knots, respectively; 
    the second row shows the corresponding 3D surfaces.
    }
    \label{fig:ubar_drainage}
\end{figure}

\begin{figure}[t]
  \centering
  \includegraphics[width=0.48\textwidth,trim=10mm 10mm 30mm 10mm, clip]{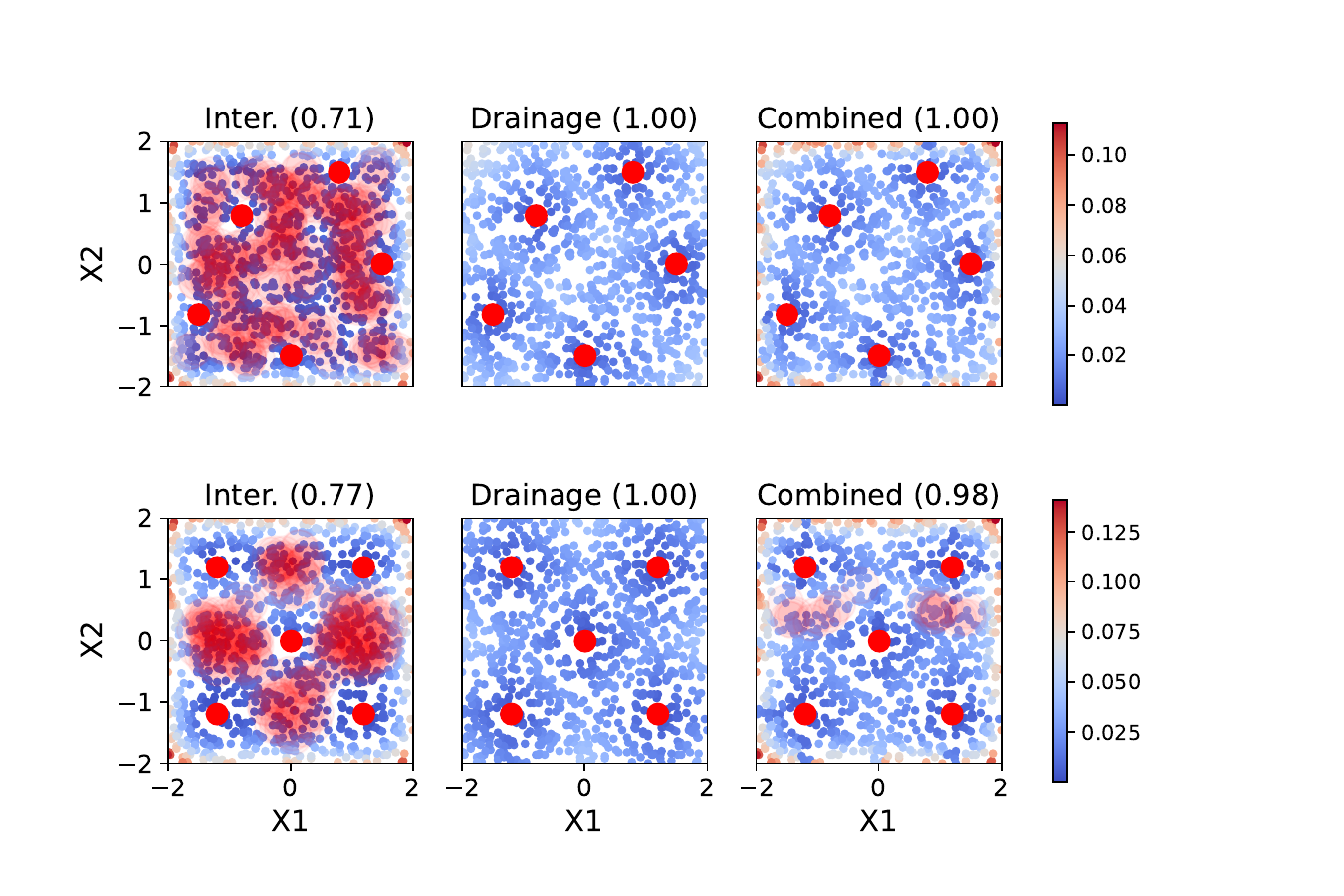}
  \caption{Distance-awareness violations for two knot configurations. 
  Red dots mark the real knots, shaded regions indicate violations, and parentheses report SDA~\eqref{eq:probability_distance_aware}.
  Drainage substantially reduces the violations of the interpolation-error bound.}
  \label{fig:violation}
\end{figure}
 
\begin{figure}[t]
  \centering
  \includegraphics[width=\columnwidth,trim=0 5mm 0 0,clip]{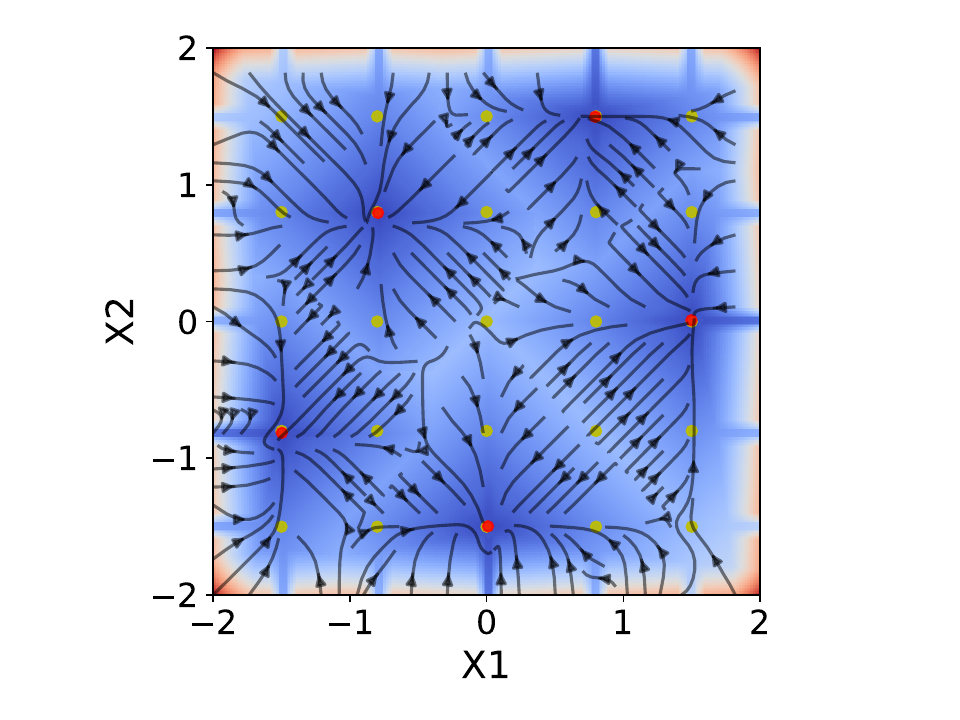}
  \vspace{-5mm}
  \caption{Streamlines of $-\nabla u_{\mathrm{da}}$ over the combined
  uncertainty field. Every streamline descends to a real knot,
  visualizing the monotone path to the training data guaranteed by
  Theorem~\ref{thm:restored-da}.}
  \label{fig:vfield}
\end{figure}
 
\begin{figure}[t]
  \centering
  \includegraphics[width=0.48\textwidth]{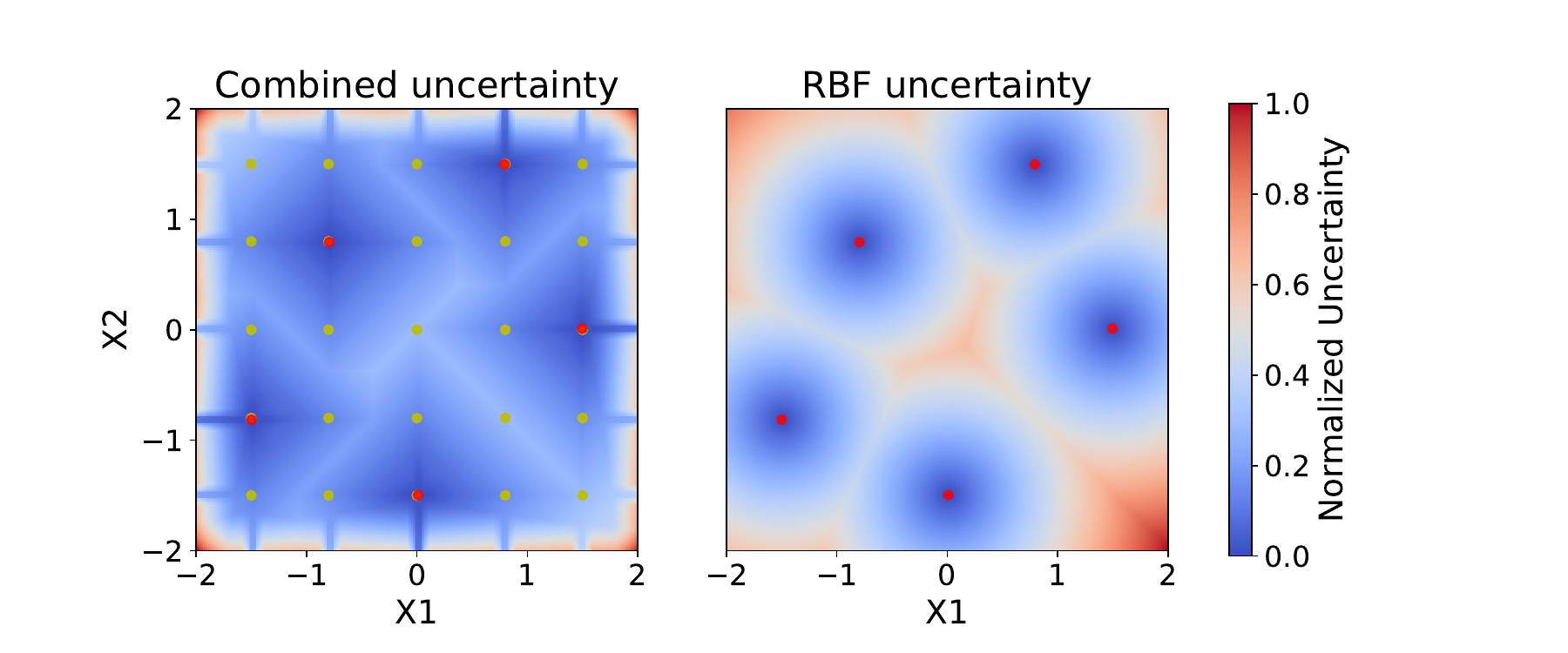}
  \vspace{-3mm}
  \caption{Combined drainage estimate (left) vs.\ a RBF-kernel
  GP (right), each normalized to $[0,1]$. Both are low at the real
  knots and rise away from them; drainage reproduces the GP's
  distance-aware structure without the spurious interior minima of raw
  DAREK.}
  \label{fig:rbf}
\end{figure}

\begin{figure}[t]
  \centering
  \includegraphics[width=0.35\textwidth]{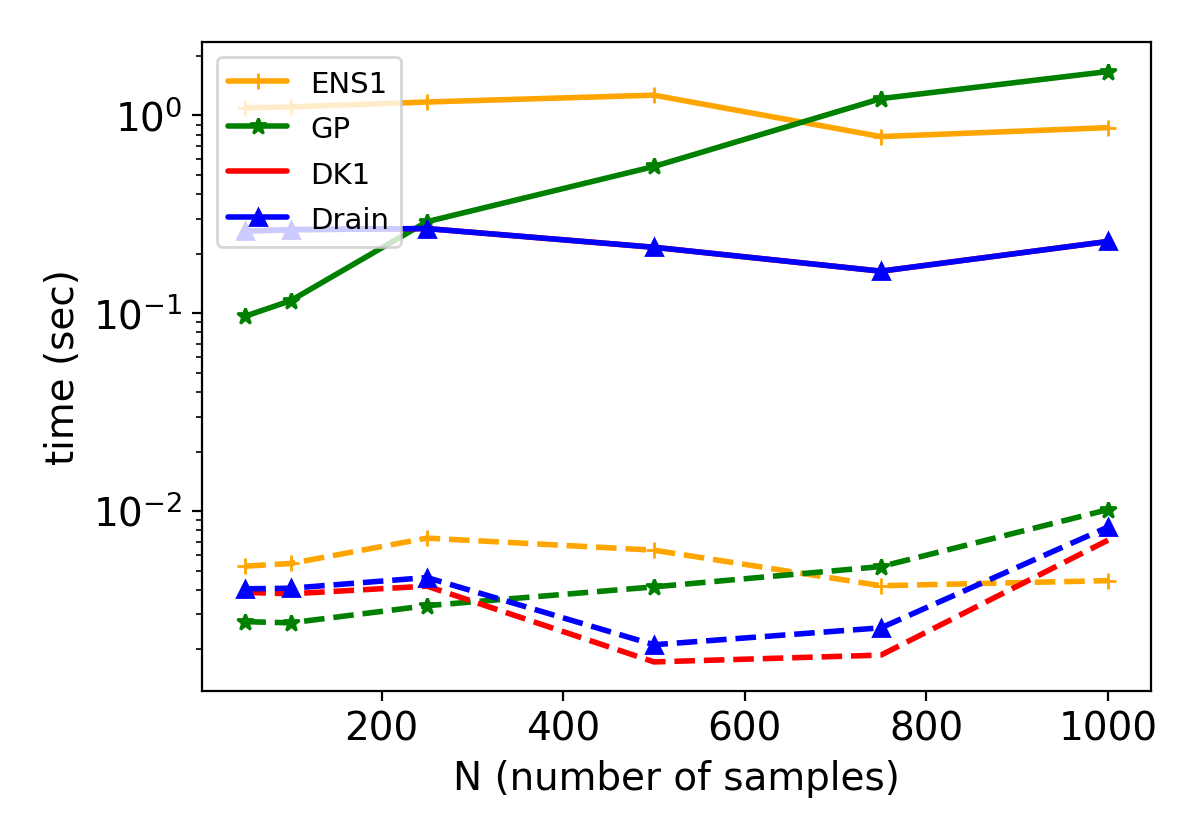}
  \caption {Wall-clock time vs. number of training/inference points $N$ (50–1000) at fixed model capacity ($m=5$). Solid lines show training time and dashed lines show inference time. GP training time grows markedly with $N$, while DAREK, Drain, and the KAN ensemble remain essentially flat.}
  \label{fig:computation-complexity}
\end{figure}

Finally, we compare wall-clock cost across methods as the number of training/inference points $N$ grows from 50 to 1000 (Fig.~\ref{fig:computation-complexity}), with model capacity held fixed at m=5 real knots. DAREK and Drain reuse the same trained model — drainage is a post-hoc correction applied at inference (Eq.~\ref{eq:combined-error}) — so their fit-time curves coincide exactly, confirming drainage adds no training overhead. The GP, trained here on the full N-point set, shows fit time increasing markedly with N, while DAREK, Drain, and the ensemble remain nearly flat since their capacity does not scale with N; at inference, DAREK/Drain's per-query cost stays close to constant, consistent with querying a fixed set of m knots regardless of dataset size.
\section{Conclusion}
We identified and formalized the \textbf{fictitious-knot problem} in high-dimensional spline networks, showing that it arises because a KAN processes each coordinate independently; the knots of its per-axis splines recombine into a grid $\mathcal{G}$ of $m^{n}$ apparent knots, of which only the fraction $m^{1-n}$ are real knots while the remaining points are fictitious knots.
The bottom-up DAREK bound vanishes on this entire grid, so it may report low uncertainty at fictitious knots that lie arbitrarily far from any data, and may therefore violate distance-awareness as the input dimension increases.
To repair this, we introduced \textbf{drainage uncertainty}, a geometric correction that routes the uncertainty along a monotonically decreasing path toward the nearest real knot, restoring a distance-aware estimate while leaving the trained network untouched. 
We showed that the corrected region occupies a fraction $1-(1-r)^{n}$ of the input space, and characterized the trade-off between the drainage thickness and input dimension. 
On a 2-D synthetic benchmark and a 100-dimensional face task, drainage raised sampled distance-awareness from roughly 85\% to 98–99\%, matching a Gaussian process at a fraction of the cost.

\textbf{Limitations.} Drainage restores the geometry of the uncertainty estimate, not a tighter worst-case bound: the drainage term $u_d$ is a linear ramp rather than a worst-case error bound, so while the combined estimate remains a valid enclosure (it never reports less than the underlying DAREK bound), its tightness in the drainage regions is not guaranteed. 
The volume guarantee, moreover, is not free in the limit: for any \textbf{fixed relative thickness} $r>0$, the drained fraction $1-(1-r)^{n}$ tends to one as $n\to\infty$, as a consequence of the product geometry.
Keeping this fraction below a fixed target requires the thickness to shrink with dimension, $\epsilon=\calO(a/(mn))$, which in turn increases the sufficient gain required by (C2).
The restored-monotonicity guarantee depends on the gain $\eta$ exceeding a sufficient threshold that scales with the worst-case DAREK uncertainty on each cell, and is therefore conservative; in practice (Sec.~\ref{sec:ablation}) far smaller gains suffice. 
Finally, our analysis targets the first layer, where the knots are tied directly to the input data, and distance-awareness is measured in input space; composition across deeper layers propagates the corrected estimate but is not separately analyzed here, and the sampled metric (SDA) averages over uniformly drawn test points, which can understate a failure that is concentrated near the fictitious-knot grid.

\textbf{Future work.} Several directions follow naturally. The ramp could be replaced by a Lipschitz-based extension anchored at the real knots, yielding a drainage term that is itself a valid worst-case bound. Pairing drainage with conformal calibration would add distribution-free coverage on top of its geometric monotonicity, combining the per-point distance-awareness that conformal prediction lacks with the marginal guarantees it provides. 
Adapting the thickness $\epsilon$ and gain $\eta$ per region, rather than globally, would tighten the correction further. A further direction is to connect these distance-aware approximation-error bounds with verification methods such as CROWN, which bound the output variation of a fixed learned network over prescribed input regions, to study how the two complementary guarantees can be combined.

\bibliographystyle{IEEEtran}
\bibliography{bib/main}

\end{document}